\documentclass[11pt]{article}

\usepackage[preprint]{acl}
\usepackage{amssymb}
\usepackage{amsthm}
\newtheorem{proposition}{Proposition}
\usepackage{times}
\usepackage{latexsym}
\usepackage[most]{tcolorbox}
\usepackage{algorithm}
\usepackage{algpseudocode}
\usepackage{booktabs}
\usepackage{placeins}
\usepackage[table]{xcolor}
\definecolor{ourspink}{RGB}{252,228,236}
\definecolor{compspink}{RGB}{253,241,245}
\definecolor{modelband}{RGB}{237,233,247}
\usepackage[T1]{fontenc}

\usepackage{microtype}

\usepackage{inconsolata}

\usepackage{graphicx}

\usepackage{amsmath}

\usepackage[most]{tcolorbox}
\usepackage{xcolor}

\definecolor{rqblue}{HTML}{2563EB}
\definecolor{rqpurple}{HTML}{7C3AED}
\definecolor{rqgreen}{HTML}{059669}

\newtcolorbox{takeawaybox}[1]{
    colback=blue!5,
    colframe=blue,
    boxrule=0.8pt,
    arc=2mm,
    left=2.5mm,
    right=2.5mm,
    top=1mm,
    bottom=1.8mm,
    before skip=5pt,
    after skip=5pt,
    fonttitle=\bfseries,
    title=#1,
    coltitle=blue,
    colbacktitle=blue!5,
    attach title to upper,
    after title={\par\vspace{0.5mm}}
}

\title{Different Tasks, Different Timesteps: Temporal Task-Aware Quantization for Diffusion Language Models}

\author{
  Ofek Avivi\textsuperscript{1,2,*} \quad
  Eliron Rahimi\textsuperscript{1,2,*} \quad
  Avi Mendelson\textsuperscript{1} \quad
  Chaim Baskin\textsuperscript{2}
\\[4pt]
  \normalfont\normalsize \textsuperscript{1}Technion -- Israel Institute of Technology
\\
  \normalfont\normalsize \textsuperscript{2}INSIGHT Lab, Ben-Gurion University of the Negev
\\[2pt]
  \normalfont\small \textsuperscript{*}Equal contribution
  \\ ~\href{https://ofekavivi.github.io/T2Quant/}{Project Page \includegraphics[height=1em]{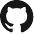}}
}

\usepackage{xspace}
\newcommand{\method}{T$^2$Quant\xspace}
\begin{document}
\maketitle
\begin{abstract}
Diffusion Language Models (DLMs) introduce unique quantization challenges due to their iterative denoising process. In this work, we reveal that quantization sensitivity is both task- and time-dependent, with different downstream tasks exhibiting distinct precision requirements across denoising steps. We further show that these patterns can be predicted without ground-truth labels or task-specific evaluation metrics. Building on these findings, we introduce \textbf{\emph{\method{}}}, a task and time-aware mixed-precision quantization framework that uses a label-free sensitivity signal to optimally allocate activation precision across denoising steps under a fixed bit-width budget. Extensive experiments across three DLMs and six benchmarks demonstrate that \textbf{\emph{\method{}}} consistently improves the average accuracy of existing quantization methods, achieving gains of up to 6.7 percentage points under the same budget.

\end{abstract}

\section{Introduction}

Diffusion Language Models (DLMs) have emerged as a compelling alternative to autoregressive (AR) language models, generating text through iterative denoising with bidirectional context and parallel token prediction \citep{nie-etal-2025-llada,ye-etal-2025-dream}. However, iterative denoising requires repeated model evaluations, making efficient inference increasingly important.

Quantization is a widely adopted technique for reducing the memory footprint and computational cost of large language models. However, methods developed for AR models may not directly transfer to DLMs due to their distinct step-wise decoding dynamics. Recent studies have identified DLM-specific quantization challenges, including activation outliers, timestep-dependent distributions, and accumulated quantization errors \citep{lin-etal-2026-quantmeetsdllms,xu-etal-2025-dllmquant,yan-etal-2026-starquant,huang-etal-2026-faircalib}. While these methods improve quantization robustness, they generally maintain a fixed activation precision throughout the denoising process. 

\begin{figure}[t]
\centering
\includegraphics[width=\linewidth]{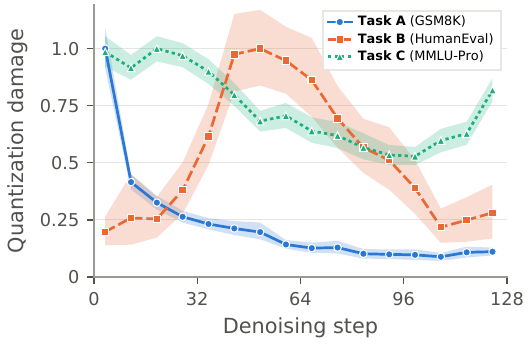}
\caption{\textbf{Task-dependent temporal sensitivity to activation quantization.}
Each curve shows, per eight-step denoising window, the KL divergence
$D_{\mathrm{KL}}(p_{\mathrm{fp}} \,\|\, p_{\mathrm{q}})$ summed over the positions committed
in that window that the task's scoring function reads.}
\label{fig:temporal-sensitivity}
\end{figure}

\begin{figure*}[t]
    \centering
    \includegraphics[width=\linewidth]{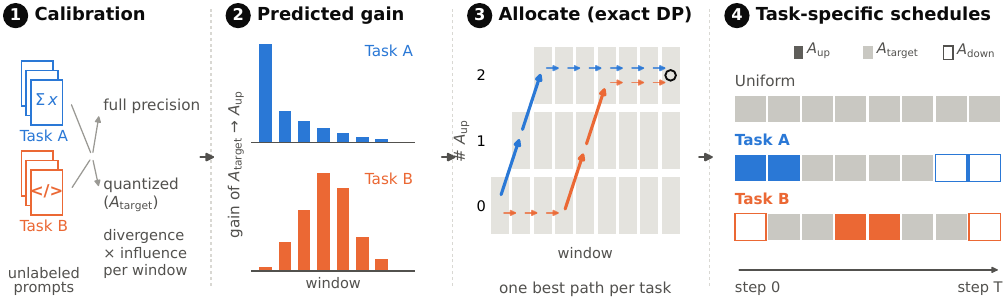}
    \caption{\textbf{\method} overview. (1) A few unlabeled prompts per task are profiled with the full-precision and quantized models. (2) Divergence $\times$ influence gives the predicted gain of upscaling each denoising window. (3) An exact dynamic program selects the windows to upscale under the bit budget. (4) The result is a task-specific schedule at the same average precision as uniform $A_{\mathrm{target}}$.}
    \label{fig:method-overview}
\end{figure*}

Meanwhile, temporal mixed-precision allocation has been explored in image diffusion models \citep{lewandowski-etal-2026-tmpdiff}, and recent studies on DLMs have highlighted how step-wise dynamics, intermediate predictions, and token commitment decisions influence generation behavior \citep{zhang2026differencestextgenerateddiffusion,rahimi2026stepwiserefusaldynamicsautoregressive, shlomi2026spotenhancingspeculativedecoding}. Nevertheless, how quantization sensitivity varies jointly across downstream tasks and denoising steps in DLMs, and how these variations can guide effective mixed-precision allocation, remains largely unexplored.

To address this gap, we systematically investigate the sensitivity of DLMs to activation quantization across different tasks and denoising steps. Our analysis reveals that \emph{precision importance is both task-dependent and time-dependent~(Figure~\ref{fig:temporal-sensitivity})}: different tasks exhibit distinct temporal sensitivity patterns, with some benefiting more from higher precision during early denoising steps and others exhibiting greater sensitivity at intermediate stages. Furthermore, we find that the benefit of additional precision can be predicted without ground-truth labels or downstream metric evaluations. Specifically, we develop a label-free sensitivity signal that combines quantization-induced divergence with the downstream influence of each denoising stage, enabling the identification of task-specific precision requirements using only unlabeled calibration prompts.

Building on these findings, we introduce \textbf{\method} (\textbf{Task and Time-Aware Quantization}), a mixed-precision framework that allocates activation precision across denoising steps according to task-specific sensitivity patterns~(Figure~\ref{fig:method-overview}). \method leverages the proposed sensitivity signal to determine an optimal precision schedule under a fixed average bit-width budget through budget-constrained dynamic programming. Our framework is complementary to existing post-training quantization methods and can be integrated with different quantization approaches without modifying their underlying quantizers. Extensive experiments across three DLMs and six benchmarks demonstrate consistent improvements in average accuracy under a fixed budget, outperforming the strongest evaluated quantized baselines by up to  6.7 percentage points. Our main contributions are summarized as follows:

\begin{itemize}
\item \textbf{Task and time-dependent precision sensitivity.} We reveal distinct temporal precision sensitivity patterns across downstream tasks in DLMs.

    \item \textbf{Label-free sensitivity prediction.} We predict precision importance using internal model signals, without ground-truth labels or task-specific evaluation metrics.

    \item \textbf{Task and time-aware mixed-precision quantization.} We introduce \method, which optimally allocates activation precision under a fixed bit budget and improves existing quantization methods.
\end{itemize}

\section{Related Work}

\paragraph{DLM Quantization.}
DLMs such as LLaDA and Dream generate text by iterative masked denoising with bidirectional context rather than left-to-right decoding \citep{nie-etal-2025-llada, ye-etal-2025-dream}, which raises quantization challenges absent from autoregressive LLMs. DLLMQuant introduces diffusion-aware calibration and quantization \citep{xu-etal-2025-dllmquant}; STaR-Quant uses state-guided transformations and temporal compensation \citep{yan-etal-2026-starquant}; FAIR-Calib targets fragile denoising states \citep{huang-etal-2026-faircalib}; and Quant-dLLM explores sensitivity-based mixed precision \citep{zhang-etal-2026-quantdllm}. These methods show that diffusion state and temporal error propagation matter, yet all keep a fixed activation bit-width across the denoising trajectory. \emph{We instead treat activation precision as a resource allocated over denoising steps under a fixed average bit budget}.

\paragraph{Task-Aware Quantization.}
Task-Aware Quantization (TAQ) estimates layer importance from a small task-specific calibration set and assigns higher weight precision to task-relevant layers under a fixed bit budget \citep{levi-etal-2026-taq}, and Task-Circuit Quantization (TaCQ) preserves task-relevant weight circuits, identified from quantization-induced weight changes and gradient-based impact estimates, at higher precision \citep{xiao-etal-2025-tacq}. Both show that the best allocation of a quantization budget depends on the task, but they allocate precision \emph{spatially}, across layers or weights of autoregressive LLMs. \emph{In contrast, \method{} explores temporal task dependence, allocating activation precision across denoising steps according to task-specific sensitivity.}

\paragraph{Temporal Precision Scheduling.}
Temporal variation is a recognized factor in \emph{image diffusion} quantization: Q-Diffusion uses timestep-aware calibration for shifting activation distributions, and correction methods compensate for errors accumulated across denoising steps \citep{li-etal-2023-qdiffusion}. Most closely related, TMPDiff assigns different precisions to different denoising timesteps in \emph{image diffusion models}, modeling the final latent deviation as an approximately additive sum of per-timestep errors and searching over schedules built from single-timestep sensitivities \citep{lewandowski-etal-2026-tmpdiff}. \emph{DLMs differ qualitatively: generation consists of discrete, often irreversible token commitments, latent reconstruction error does not directly reflect downstream correctness}, and a timestep's importance depends on which task-relevant tokens it commits. Rather than a single global schedule, \emph{\method allocates activation precision per task}, linking where quantization disrupts task-relevant generation to where additional precision should be spent.

\section{Task-Dependent Precision Sensitivity Across Diffusion Steps}
\label{sec:analysis}
The iterative generation process of DLMs raises a fundamental question for quantization: \emph{when does numerical precision matter most?} Each denoising step operates on a different partially generated state, with token commitments potentially influencing subsequent predictions \citep{nie-etal-2025-llada,ye-etal-2025-dream}. Consequently, the impact of quantization may depend not only on the denoising step but also on the downstream task. Understanding these variations, and whether they can be predicted, motivates three research questions:

\begin{itemize}
\item \textbf{RQ1:} Where does quantization sensitivity occur in time?
\item \textbf{RQ2:} Is temporal sensitivity task-dependent?
\item \textbf{RQ3:} Can we predict where additional precision is most valuable?
\end{itemize}

We address RQ1 and RQ2 through temporal divergence analysis (Section~\ref{sec:analysis-diagnostic}), and RQ3 through a label-free precision-allocation proxy (Section~\ref{sec:analysis-signal}).

\paragraph{Experimental setup.}
We evaluate LLaDA-8B-Instruct \citep{nie-etal-2025-llada} on three tasks with distinct output structures: GSM8K \citep{cobbe-etal-2021-gsm8k}, HumanEval \citep{chen-etal-2021-humaneval}, and MMLU-Pro \citep{wang-etal-2024-mmlupro}. We use 512-token generations over 128 denoising steps with a single 512-token block, partitioning the trajectory into 16 non-overlapping windows of eight steps. We use the standard low-confidence remasking sampler of \citet{nie-etal-2025-llada} with temperature 0 and no classifier-free guidance. The quantized model uses QuaRot-style rotations \citep{ashkboos-etal-2024-quarot}, 4-bit GPTQ weights \citep{frantar-etal-2023-gptq}, and dynamic per-token 4-bit activations. The target precision is $\mathrm{A}_{\mathrm{target}} = 4$ bits and the upscaled precision is $\mathrm{A}_{\mathrm{up}} = 8$ bits. Further details are given in Appendix~\ref{app:experimental_setup}.

\subsection{Temporal Sensitivity Diagnostic}
\label{sec:analysis-diagnostic}

\paragraph{Temporal sensitivity.}
Figure~\ref{fig:temporal-sensitivity} shows that quantization sensitivity is highly non-uniform across
the denoising trajectory. A uniform distribution would assign each of the 16 eight-step windows 6.25\%
of the task-relevant divergence. Instead, on GSM8K the first window alone carries 27\% (4.4$\times$
uniform) and the first quarter 55\%, while on HumanEval five mid-trajectory windows (steps 32--71) carry
52\%. Even MMLU-Pro, the most evenly spread task, varies by nearly 2$\times$ between its most and least
sensitive windows. This concentration suggests that activation precision need not be allocated uniformly over time and motivates selectively increasing precision at the most sensitive stages, as explored for image diffusion models \citep{lewandowski-etal-2026-tmpdiff}.
\begin{takeawaybox}
\textbf{\color{blue!55!black}Key Takeaway} Quantization sensitivity varies across denoising steps.
\end{takeawaybox}

\paragraph{Task-dependent sensitivity.}
The location and spread of these sensitive regions vary substantially across tasks. GSM8K is strongly
front-loaded, with its peak in the first 8 denoising steps and a rapid decay thereafter. MMLU-Pro is
also early, but broad: its first four windows (steps 0--31) are all within 10\% of its peak, after which
sensitivity declines gradually. In contrast, HumanEval is insensitive early and exhibits a broad
mid-trajectory peak, roughly around steps 40--63. Thus, a single task-independent temporal precision
schedule is unlikely to suit all tasks. Whether these sensitivity profiles can predict where additional
precision actually improves downstream performance is examined in Section~\ref{sec:analysis-signal}.
\begin{takeawaybox}
\textbf{\color{blue!55!black}Key Takeaway} Sensitivity patterns differ across tasks, challenging
fixed temporal precision schedules.
\end{takeawaybox}

\begin{figure*}[t]
    \centering
    \includegraphics[width=\linewidth]{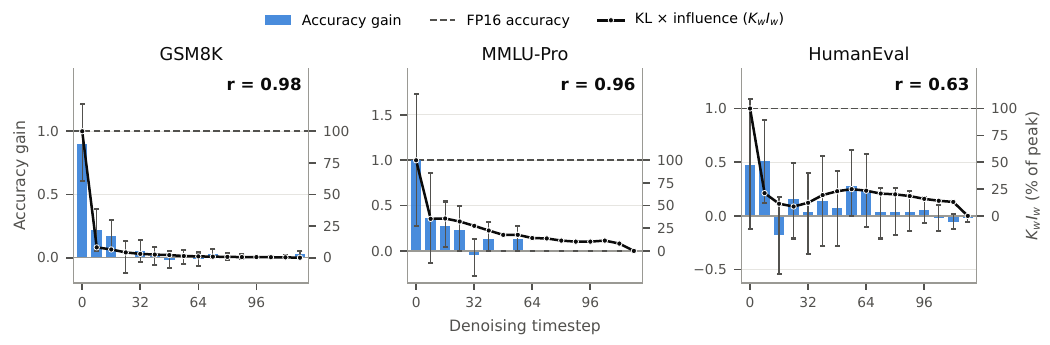}
    \caption{\textbf{KL $\times$ influence predicts where additional precision pays off.} Bars show the accuracy gain from upscaling a single denoising window from $\mathrm{A}_{\mathrm{target}}$ to $\mathrm{A}_{\mathrm{up}}$, normalized by the full-precision gap (dashed line, 1.0 = full-precision accuracy), with 95\% bootstrap intervals. The black curve shows the label-free proxy $K_w I_w$, normalized to its peak window. $r$ is the Pearson correlation between the two across windows.}
    \label{fig:proxy-correlation}
\end{figure*}

\subsection{A Label-Free Signal for Precision Allocation}
\label{sec:analysis-signal}

We now address RQ3: identifying, without downstream labels, which denoising windows benefit most from additional activation precision.

\paragraph{Denoising-window upscaling.}
Let $\mathrm{A}_{\mathrm{target}}$ and $\mathrm{A}_{\mathrm{up}}$ denote the target and higher activation precisions. We upscale each window $w$ individually, keeping the remaining steps at $\mathrm{A}_{\mathrm{target}}$ and weights fixed. The resulting gain is $\Delta\mathcal{M}_w = \mathcal{M}_{\mathrm{A}_{\mathrm{up}}(w)} - \mathcal{M}_{\mathrm{A}_{\mathrm{target}}}$, where $\mathcal{M}$ is the task metric.

\paragraph{A label-free proxy.}
Local divergence alone cannot capture a window's importance: its commitments must also influence subsequent denoising \citep{huang-etal-2026-faircalib, xu-etal-2025-dllmquant}. Along the uniform $\mathrm{A}_{\mathrm{target}}$ trajectory,
we measure full-precision divergence at committed positions $C_t$:
\begin{equation}
\label{eq:kl}
K_w = \sum_{t \in w}\sum_{j \in C_t} D_{\mathrm{KL}}\!\left(p_{\mathrm{fp}}(\cdot \mid x_t)_j \,\Vert\, p_{\mathrm{A}_{\mathrm{target}}}(\cdot \mid x_t)_j\right).
\end{equation}
$K_w$ measures how much additional precision can change the window's decisions. To measure how consequential those decisions are, let $x_{\mathrm{before}}$ and $x_{\mathrm{after}}$ be the states immediately before and after window $w$, and $F_w$ the positions still masked afterwards. The window's downstream influence is the change in the full-precision model's beliefs about these positions:
\begin{equation}
\label{eq:analysis-influence}
I_w = \sum_{j \in F_w} D_{\mathrm{KL}}\!\left(p_{\mathrm{fp}}(\cdot \mid x_{\mathrm{after}})_j \,\Vert\, p_{\mathrm{fp}}(\cdot \mid x_{\mathrm{before}})_j\right).
\end{equation}
Both are computed from the $\mathrm{A}_{\mathrm{target}}$ trajectory alone, without labels.

\begin{proposition}
\label{prop:bound}
For an evaluation example, let $\Pi_{\mathrm{up}}$ and $\Pi_{\mathrm{tg}}$ be the distributions of window $w$'s commitments (token values
and positions) when the window runs at $\mathrm{A}_{\mathrm{up}}$ and at $\mathrm{A}_{\mathrm{target}}$ from the same state, let
$J_w = D_{\mathrm{KL}}(\Pi_{\mathrm{up}} \Vert \Pi_{\mathrm{tg}})$, and let $R_w$ be the range of the expected final score over the
window's possible commitments. Then
\[
\left|\Delta \mathcal{M}_w\right|^{2} \le \tfrac{1}{2}\, \mathbb{E}[J_w]\, \mathbb{E}[R_w^{2}],
\]
with expectations over evaluation examples.
\end{proposition}

The bound (proof in Appendix~\ref{app:bound-proof}) separates the effect of upscaling into commitment perturbation ($J_w$) and downstream score sensitivity ($R_w$). Since neither is available from a single unlabeled trajectory, we use $K_w$ and $I_w$ as respective surrogates. Mirroring the bound, we rank windows by
\[
S_w = \mathbb{E}[K_w]\, \mathbb{E}[I_w].
\]
This is a motivated proxy, not a formal bound; we evaluate its predictive ability empirically.

\paragraph{Results.}
Figure~\ref{fig:proxy-correlation} compares $S_w$
with the measured gain $\Delta\mathcal{M}_w$.
On GSM8K and MMLU-Pro, the proxy correctly identifies
the most beneficial window, recovering 90\% and 100\%
of the full-precision gap, respectively.
The proxy closely tracks window-level gains
($r=0.98$ and $0.96$), including beyond the
highest-ranked window. On HumanEval, both the proxy
and observed gains concentrate in the first two windows, although measurement noise limits correlation ($r=0.63$). Influence is essential: divergence alone incorrectly favors HumanEval's final window, whose commitments cannot affect future tokens. Finally, just 16 unlabeled calibration prompts yield proxy estimates closely matching those computed from all prompts ($r\geq 0.96$ across tasks; Figure~\ref{fig:calib-size}).

\begin{takeawaybox}
\textbf{\color{blue!55!black}Key Takeaway} Precision importance can be predicted from internal model signals using only a few unlabeled calibration prompts, without requiring task-specific evaluation metrics.
\end{takeawaybox}


\section{Method}
\label{sec:method}
Section~\ref{sec:analysis-diagnostic} establishes that activation-quantization sensitivity varies across denoising steps and downstream tasks. Section~\ref{sec:analysis-signal} further shows that a label-free signal combining quantization divergence and downstream influence can identify windows where additional precision is most beneficial. Building on these findings, we introduce \textbf{\method}, which extends single-window sensitivity analysis to a task-specific activation-precision schedule under a fixed budget. Rather than designing a new quantizer, \method uses unlabeled task-specific prompts to estimate the value of each available precision and optimizes their allocation across the denoising trajectory.

\subsection{Problem Setup}
\label{sec:method-setup}
\paragraph{Notation and objective.}
Let $\tau$ denote a target task and $\mathcal D_\tau=\{s_i\}_{i=1}^{N}$ a representative calibration set of prompts. A diffusion language model generates through $T$ denoising steps, with partially generated state $x_{i,t}$ immediately before step $t$. Let $C_{i,t}$ be the positions committed at that step. Following the window-based analysis of Section~\ref{sec:analysis-diagnostic}, we partition normalized denoising progress into $K$ ordered, non-overlapping temporal windows $\{\mathcal W_k\}_{k=1}^{K}$, so a schedule can be applied across different step counts. Given a finite set of supported activation precisions $\mathcal B$ and a uniform reference precision $b_0\in\mathcal B$, our task-specific schedule is $\mathbf b_\tau=(b_1,\ldots,b_K)\in\mathcal B^K$, where all denoising steps in $\mathcal W_k$ use $b_k$. Denote the additive resource cost of assigning precision $b$ to window $k$ by $c_k(b)$ (e.g., the window's activation work multiplied by $b$); the reference budget is $C_0=\sum_k c_k(b_0)$. Ideally, the schedule maximizes the downstream task metric $\mathcal M_\tau(\mathbf b)$ subject to $\sum_k c_k(b_k)=C_0$. Because evaluating this objective for candidate schedules would require labeled, closed-loop generations, we instead construct a label-free surrogate.

\paragraph{Calibration assumption.}
As in task-aware PTQ~\citep{levi-etal-2026-taq,xiao-etal-2025-tacq}, we assume access to representative unlabeled calibration prompts from the target task. Following the label-free analysis of Section~\ref{sec:analysis}, \method{} estimates precision importance using internal model signals, without ground-truth answers or task-specific evaluation metrics.

\paragraph{Quantizer-agnostic scheduling.}
\method{} allocates activation precision across denoising steps without modifying the underlying quantizer or sampler. It can therefore be integrated with existing PTQ methods, including DLLMQuant~\citep{xu-etal-2025-dllmquant}, FAIR-Calib~\citep{huang-etal-2026-faircalib}, and STaR-Quant~\citep{yan-etal-2026-starquant}.

\subsection{Label-Free Calibration}
\label{sec:method-calibration}
\paragraph{Local quantization divergence.}
Following the committed-position KL signal in Eq.~\eqref{eq:kl}, we extend the analysis to every candidate activation precision $b\in\mathcal B$. For each $s_i\in\mathcal D_\tau$, we generate a full-precision reference trajectory and save its states $x_{i,t}$ and committed sets $C_{i,t}$. We replay each saved state under every $b$, keeping the input state fixed across precisions. Let $p^{\mathrm{fp}}_{i,t,j}$ and $p^{(b)}_{i,t,j}$ denote the resulting full-precision and quantized predictive distributions at position $j$. The average divergence over all positions committed in window $k$ is
\begin{equation}
\label{eq:local-kl}
\kappa_{\tau,k}(b)=
\frac{\displaystyle\sum_i\sum_{t\in\mathcal W_k}
\sum_{j\in C_{i,t}}
D_{\mathrm{KL}}\!\left(p^{\mathrm{fp}}_{i,t,j}\Vert p^{(b)}_{i,t,j}\right)}
{\displaystyle\sum_i\sum_{t\in\mathcal W_k}|C_{i,t}|}.
\end{equation}
This extends the single-precision divergence in Section~\ref{sec:analysis-signal} to multiple precision levels, normalized per committed position. For scheduling, the replay states come from the full-precision reference trajectory, whereas Section~\ref{sec:analysis-signal} evaluates its diagnostic on the uniform-precision trajectory. Holding each replay state fixed isolates local quantization effects from changes in the generated trajectory; no task-specific token selection is required.

\paragraph{Downstream influence.}
As demonstrated in Section~\ref{sec:analysis-signal}, local divergence alone does not capture how a window's commitments affect subsequent predictions. We therefore reuse the downstream-influence measure of Eq.~\eqref{eq:analysis-influence}, evaluated on the saved reference trajectory. For each prompt $i$ and window $k$, let $x^-_{i,k}$ and $x^+_{i,k}$ be the states immediately before and after the window, and let $M_{i,k}$ contain the positions still masked afterwards. Applying Eq.~\eqref{eq:analysis-influence} to these states gives $\iota_{i,k}$; we average across prompts as $\bar\iota_{\tau,k}=N^{-1}\sum_i\iota_{i,k}$. This measures the window's downstream leverage using internal model predictions, without requiring task labels.
\definecolor{ourspink}{RGB}{252,228,236} \definecolor{modelband}{RGB}{237,233,247} \definecolor{avgcol}{RGB}{253,237,225}

\providecommand{\method}{T$^2$Quant}
\begin{table*}[t]
\centering\small\setlength{\tabcolsep}{4.5pt}
\caption{W4A4 accuracy (\%) of three diffusion-LM quantization methods, each with and without \method{}, on three masked-diffusion LMs and six benchmarks at the same 4.00 average activation bits (bold: best, underline: second-best quantized result).}
\label{tab:main_results}
\begin{tabular}{lccccccc}
\toprule
\textbf{Method} & \cellcolor{avgcol}\textbf{Avg.} & \multicolumn{6}{c}{\textbf{Evaluation Benchmarks} $\uparrow$}\\
\cmidrule(lr){3-8}
 & \cellcolor{avgcol} & \textbf{GSM8K} & \textbf{Hum.} & \textbf{ARC-C} & \textbf{MMLU-Pro} & \textbf{DROP} & \textbf{BBH}\\
\midrule
\rowcolor{modelband}\multicolumn{8}{l}{\textit{\textbf{Dream-7B-Instruct}}}\\
FP16 & \cellcolor{avgcol}44.05 & 62.50 & 19.00 & 87.30 & 29.69 & 37.89 & 27.93\\
RTN & \cellcolor{avgcol}0.00 & 0.00 & 0.00 & 0.00 & 0.00 & 0.00 & 0.00\\
DLLMQuant & \cellcolor{avgcol}33.99 & 35.55 & 9.00 & 79.30 & 24.80 & 38.09 & 17.19\\
FAIR-Calib & \cellcolor{avgcol}34.81 & 44.92 & 4.00 & 81.64 & 20.70 & 41.60 & 16.02\\
STaR-Quant & \cellcolor{avgcol}38.55 & 46.09 & 11.00 & 83.01 & 26.17 & \textbf{46.68} & 18.36\\
Random schedule & \cellcolor{avgcol}19.61 & 15.82 & 4.00 & 54.30 & 13.67 & 20.31 & 9.57\\
\rowcolor{ourspink}\textbf{DLLMQuant + \method{} (Ours)} & \cellcolor{avgcol}\underline{40.69} & 53.52 & \underline{17.00} & \underline{84.18} & \textbf{29.69} & 40.43 & 19.34\\
\rowcolor{ourspink}\textbf{FAIR-Calib + \method{} (Ours)} & \cellcolor{avgcol}40.61 & \underline{57.03} & 12.00 & 82.81 & \underline{26.95} & 39.45 & \textbf{25.39}\\
\rowcolor{ourspink}\textbf{STaR-Quant + \method{} (Ours)} & \cellcolor{avgcol}\textbf{42.10} & \textbf{58.79} & \textbf{20.00} & \textbf{84.77} & 24.02 & \underline{42.77} & \underline{22.27}\\
\midrule
\rowcolor{modelband}\multicolumn{8}{l}{\textit{\textbf{LLaDA-1.5-8B}}}\\
FP16 & \cellcolor{avgcol}47.73 & 57.03 & 26.00 & 82.62 & 37.30 & 40.43 & 42.97\\
RTN & \cellcolor{avgcol}0.13 & 0.00 & 0.00 & 0.59 & 0.00 & 0.20 & 0.00\\
DLLMQuant & \cellcolor{avgcol}39.09 & 35.94 & 15.00 & 78.52 & 32.23 & 34.96 & 37.89\\
FAIR-Calib & \cellcolor{avgcol}42.88 & 47.07 & \underline{18.00} & 80.86 & \textbf{35.16} & 36.33 & 39.84\\
STaR-Quant & \cellcolor{avgcol}44.09 & 48.63 & \textbf{20.00} & 79.49 & \underline{34.96} & 39.26 & \textbf{42.19}\\
Random schedule & \cellcolor{avgcol}39.11 & 31.64 & \textbf{20.00} & 76.56 & 30.86 & 35.55 & 40.04\\
\rowcolor{ourspink}\textbf{DLLMQuant + \method{} (Ours)} & \cellcolor{avgcol}42.47 & 45.51 & 15.00 & \textbf{83.79} & 33.59 & 37.11 & 39.84\\
\rowcolor{ourspink}\textbf{FAIR-Calib + \method{} (Ours)} & \cellcolor{avgcol}\underline{44.82} & \underline{52.15} & 16.00 & \underline{82.81} & 34.77 & \underline{41.80} & \underline{41.41}\\
\rowcolor{ourspink}\textbf{STaR-Quant + \method{} (Ours)} & \cellcolor{avgcol}\textbf{45.22} & \textbf{55.08} & 17.00 & 82.23 & 34.57 & \textbf{42.97} & 39.45\\
\midrule
\rowcolor{modelband}\multicolumn{8}{l}{\textit{\textbf{LLaDA-8B-Instruct}}}\\
FP16 & \cellcolor{avgcol}46.32 & 52.93 & 25.00 & 81.84 & 36.52 & 39.26 & 42.38\\
RTN & \cellcolor{avgcol}0.26 & 0.39 & 0.00 & 0.98 & 0.00 & 0.20 & 0.00\\
DLLMQuant & \cellcolor{avgcol}37.35 & 31.84 & 12.00 & 77.73 & 30.66 & 34.96 & 36.91\\
FAIR-Calib & \cellcolor{avgcol}42.27 & 43.75 & 22.00 & 78.12 & \textbf{33.79} & 34.96 & \underline{41.02}\\
STaR-Quant & \cellcolor{avgcol}43.36 & 45.70 & \textbf{25.00} & 78.52 & 32.23 & \underline{38.09} & 40.62\\
Random schedule & \cellcolor{avgcol}37.72 & 30.08 & 23.00 & 75.39 & 31.25 & 27.73 & 38.87\\
\rowcolor{ourspink}\textbf{DLLMQuant + \method{} (Ours)} & \cellcolor{avgcol}40.88 & 42.19 & 15.00 & \textbf{81.84} & 29.88 & \underline{38.09} & 38.28\\
\rowcolor{ourspink}\textbf{FAIR-Calib + \method{} (Ours)} & \cellcolor{avgcol}\textbf{44.76} & \textbf{51.95} & \underline{24.00} & 80.08 & \textbf{33.79} & 36.52 & \textbf{42.19}\\
\rowcolor{ourspink}\textbf{STaR-Quant + \method{} (Ours)} & \cellcolor{avgcol}\underline{43.70} & \underline{48.83} & 21.00 & \underline{80.47} & \underline{32.62} & \textbf{38.28} & \underline{41.02}\\
\bottomrule
\end{tabular}
\end{table*}
\paragraph{Influence-weighted precision benefit.}
Following the divergence--influence proxy of Section~\ref{sec:analysis-signal}, we combine the change in local divergence with downstream influence to estimate the benefit of using precision $b$ instead of $b_0$:
\begin{equation}
\label{eq:predicted-benefit}
g_{\tau,k}(b)=\bar\iota_{\tau,k}
\left[\kappa_{\tau,k}(b_0)-\kappa_{\tau,k}(b)\right].
\end{equation}
In particular, $g_{\tau,k}(b_0)=0$; positive values favor upgrading a window, while negative values quantify the estimated cost of reducing its precision. This generalizes the window-level signal of Section~\ref{sec:analysis-signal} to both precision upgrades and downgrades. The score is a surrogate for the marginal downstream benefit, not the true change in task accuracy. Algorithm~\ref{alg:tq-calibration} (Appendix~\ref{app:algorithms}) describes the complete label-free profiling procedure.

\subsection{Optimal Budget-Constrained Scheduling}
\label{sec:method-allocation}
While Section~\ref{sec:analysis-signal} studies the benefit of upscaling one window at a time, our method must distribute precision jointly across all windows under the same total budget. Given the task-specific scores, we solve
\begin{equation}
\label{eq:allocation}
\begin{aligned}
\mathbf b_\tau^*\in\arg\max_{\mathbf b\in\mathcal B^K}\quad
 &\sum_{k=1}^{K}g_{\tau,k}(b_k)\\
\text{s.t.}\quad &\sum_{k=1}^{K}c_k(b_k)=C_0.
\end{aligned}
\end{equation}
The equality constraint compares the learned schedule with uniform precision at the same modeled budget, and the uniform schedule is always feasible. We use dynamic programming over windows and the cumulative budget. For integer-valued costs (or costs expressed in a common integer unit), let $F_k(z)$ be the largest surrogate benefit attainable by allocating budget $z$ to the first $k$ windows. With $F_0(0)=0$ and $F_0(z)=-\infty$ for $z\neq0$, the recurrence is
\begin{equation}
\label{eq:dp}
F_k(z)=\max_{\substack{b\in\mathcal B\\c_k(b)\le z}}
\left\{F_{k-1}\bigl(z-c_k(b)\bigr)+g_{\tau,k}(b)\right\}.
\end{equation}
Backtracking from $F_K(C_0)$ yields the selected schedule; the search also permits the all-$b_0$ solution when the estimated gain from reallocating precision is insufficient. Algorithm~\ref{alg:tq-dp} (Appendix~\ref{app:algorithms}) gives the complete budgeted solver.

\begin{proposition}[Optimality of temporal allocation]
\label{prop:dp-optimal}
For a finite $\mathcal B$ and nonnegative, integer-valued additive costs $c_k(b)$, the recurrence in Eq.~\eqref{eq:dp} returns a globally optimal solution of Eq.~\eqref{eq:allocation} whenever the budget is feasible.
\end{proposition}
The proof follows from the principle of optimality and is given in Appendix~\ref{app:dp-proof}. The guarantee concerns the calibrated additive surrogate, not the generally non-additive true downstream metric. Calibration produces one schedule per task; at inference, each step selects its assigned precision by a window-index lookup, with no additional diagnostic forward passes.

\begin{figure*}[h]
\centering
\includegraphics[width=\textwidth]{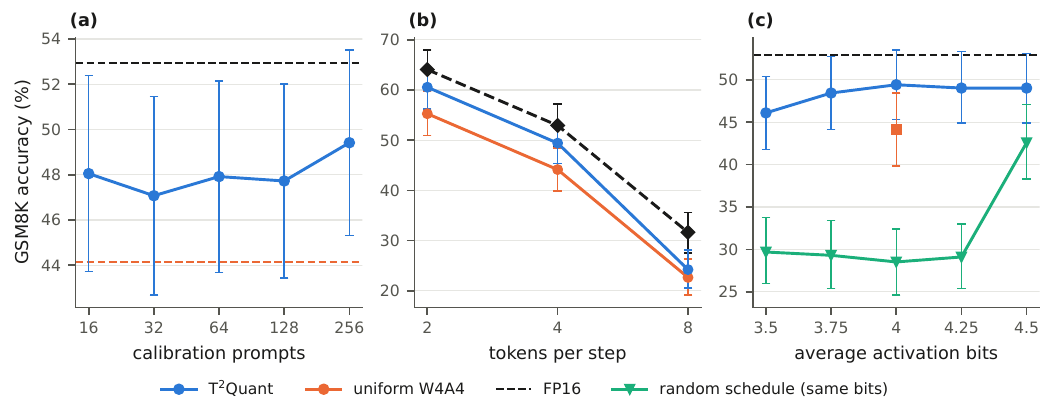}
\caption{Ablations on LLaDA-8B-Instruct, GSM8K ($n=512$; bars are 95\% intervals): (a) calibration prompts (mean over three random
subsets per size), (b) tokens per step with the schedule calibrated at 4, (c) average activation bits, with uniform W4A4 at 4 bits and,
at every budget, a random schedule with the same bits as \method{}.}
\label{fig:ablations}
\end{figure*}

\section{Experiments}
\subsection{Experimental Setup}
\label{sec:setup}

\textbf{Models.} Three open DLMs: Dream-7B-Instruct~\citep{ye-etal-2025-dream}, LLaDA-1.5-8B~\citep{zhu-etal-2026-llada} and
LLaDA-8B-Instruct~\citep{nie-etal-2025-llada}.
\textbf{Benchmarks.} GSM8K~\citep{cobbe-etal-2021-gsm8k}, HumanEval~\citep{chen-etal-2021-humaneval}, ARC-Challenge~\citep{clark-etal-2018-arc}, MMLU-Pro~\citep{wang-etal-2024-mmlupro}, DROP~\citep{dua-etal-2019-drop} and BBH~\citep{suzgun-etal-2023-challenging}, evaluated on held-out prompts that are disjoint from every calibration prompt.
\textbf{Baselines.} Full precision (FP16), round-to-nearest (RTN), and three recent post-training quantization methods for diffusion
LMs, DLLMQuant~\citep{xu-etal-2025-dllmquant}, FAIR-Calib~\citep{huang-etal-2026-faircalib} and STaR-Quant~\citep{yan-etal-2026-starquant}, each
evaluated as is and with \method{} applied on top; a random schedule with the same activation
budget controls for mixed precision itself. All quantized models and every row
uses the same sampler and the same prompts, so all comparisons are paired. Splits, sampler, quantization setup, baseline implementations and
statistical tests are given in Appendix~\ref{app:main_setup}.

\subsection{Main Results}
\label{sec:main_results}

Table~\ref{tab:main_results} reports W4A4 accuracy across three DLMs and six benchmarks, comparing DLLMQuant, FAIR-Calib, and STaR-Quant with and without \method{}. All configurations use 4-bit weights and the same average activation precision of 4.00 bits, ensuring a fair comparison under identical bit budgets.

\textbf{Consistent improvements across methods and models.}
\method{} improves average accuracy for all nine baseline/model combinations, with gains of up to 6.7 points on Dream, 3.4 on LLaDA-1.5, and 3.5 on LLaDA-8B-Instruct. Overall, 42 of 54 benchmark-level comparisons improve, with GSM8K and ARC-Challenge improving in every combination.

\textbf{Best quantized accuracy across all models.}
\method{} achieves the two highest quantized averages on every model, outperforming the strongest standalone baselines by 3.55, 1.13, and 1.40 points on Dream, LLaDA-1.5, and LLaDA-8B-Instruct, respectively. The best configurations approach FP16 accuracy, with gaps of only 2.0, 2.5, and 1.6 points.

\textbf{Substantial recovery from quantization degradation.}
The largest gains occur on Dream, where GSM8K accuracy rises from 35.6 to 53.5 with DLLMQuant and from 46.1 to 58.8 with STaR-Quant. On GSM8K, improvements also reach 9.6 points on LLaDA-1.5 and 10.4 on LLaDA-8B-Instruct. On several benchmarks, \method{} even matches or exceeds FP16 accuracy.

\textbf{Precision allocation matters.}
Random schedules with the same average bit budget perform substantially worse, falling 14--19 points below the original baselines on Dream. In contrast, \method{} achieves the best quantized average on every model, highlighting the importance of allocating precision to the appropriate denoising steps.

\subsection{Ablation Study}
\label{sec:ablations}

Figure~\ref{fig:ablations} evaluates three aspects of \method{} on LLaDA-8B-Instruct and GSM8K under W4A4 (uniform W4A4: 44.1, \method{}: 49.4, FP16: 52.9). \textbf{(a) Calibration size:} Just 16 prompts recover roughly 75\% of the full gain, with stable schedules consistently prioritizing early denoising steps. \textbf{(b) Tokens per step:} The schedule transfers from 4 to 2 tokens per step, maintaining an approximately 5-point improvement, while gains vanish at 8 tokens per step as overall accuracy deteriorates. \textbf{(c) Activation budget:} \method{} achieves 48.4\% accuracy with only 3.75 average bits, outperforming uniform W4A4 (44.1\%) despite using fewer bits. Moreover, \method{} consistently outperforms random schedules at the same bit budget, while increasing precision beyond 4 bits provides negligible additional benefit. These results demonstrate that effective precision allocation matters more than simply increasing the bit budget. Further ablations, including window size, precision levels, the set of activation widths (Table~\ref{tab:levelset_ablation}) additional quantization budgets, and calibration signals, are provided in Appendix~\ref{app:ablations}.

\section{Discussion}
Our findings demonstrate that quantization sensitivity in DLMs depends on both the downstream task and the denoising step, challenging uniform precision allocation. Building on this observation, we introduced \method{}, a task- and time-aware mixed-precision framework that predicts precision importance using unlabeled calibration prompts and allocates activation precision under a fixed bit-width budget. Across three DLMs and six benchmarks, \method{} consistently improves average accuracy when integrated with existing quantization methods, with gains of up to 6.7 percentage points under W4A4. These results highlight the importance of considering both task requirements and temporal generation dynamics when quantizing DLMs.

By improving accuracy under fixed precision budgets, \method{} offers a promising direction for making DLM deployment more computationally efficient and accessible, potentially reducing memory and energy requirements. Its compatibility with existing quantization methods may further facilitate practical adoption. However, quantization can also alter model behavior beyond benchmark accuracy, including safety-related responses~\citep{hong-etal-2024-decoding, egashira-etal-2024-exploiting}. Future work should therefore examine the effects of task- and time-aware precision allocation on safety, robustness, and real-world inference efficiency.

\clearpage
\section*{Limitations}
Although \method{} demonstrates strong results across three DLMs, six benchmarks, and three quantization methods, further evaluation of deployment-related aspects, including hardware efficiency (inference speed and energy consumption) and potential impacts on model safety, remains for future work.

\section*{LLM Usage Statement}
We used LLMs during manuscript preparation mainly as writing assistants, for grammar checking and for improving the clarity and naturalness of the text. All LLM-generated suggestions were manually reviewed and edited by the authors. LLMs did not play any important role in developing the core methods, experiments, analysis, or conclusions of this research.

\section*{Ethical Considerations}
This work aims to improve the efficiency and accessibility of diffusion language models through quantization. While quantization may affect model safety and robustness, investigating these effects is left for future work.

\clearpage
\newpage
\bibliography{references}

\clearpage
\newpage
\appendix

\section*{Appendix Overview}
\label{sec:appendix}

This appendix provides supplementary analyses, additional experimental
results, ablation studies, theoretical proofs, and implementation details
that complement the main paper. It is organized as follows:

\begin{itemize}
    \item \textbf{Section~\ref{app:additional_analysis}} presents
    further analyses and supplementary experimental results.

    \item \textbf{Section~\ref{app:ablations}} provides additional
    ablation studies to further examine the proposed method.

    \item \textbf{Section~\ref{app:proofs}} contains the theoretical
    proofs supporting the claims made in the main paper.

    \item \textbf{Section~\ref{app:algorithms}} provides detailed
    algorithms and pseudocode for the proposed method.

     \item \textbf{Section~\ref{app:experimental_setup}} details the
    experimental setup, implementation details, and evaluation protocols.
\end{itemize}

\section{Further Analyses and Supplementary Results}
\label{app:additional_analysis}

\subsection{Temporal Sensitivity Diagnostic}
\label{app:other_models}

We repeat the label-free diagnostic of Section~\ref{sec:analysis} on two further masked-diffusion LMs: LLaDA-1.5~\citep{zhu-etal-2026-llada} and Dream-v0-Instruct-7B~\citep{ye-etal-2025-dream}.
\begin{itemize}
    \item \textbf{Generation:} 512 tokens, 128 denoising steps in a single block, greedy decoding.
    \item \textbf{Quantization:} the same W4A4 recipe (rotation, GPTQ W4, per-token A4 with clip 0.7) and calibration data.
    \item \textbf{Tasks:} the same subsets, prompts and scorers.
\end{itemize}
Each model is quantized with a single rotation seed.

\begin{figure}[t]
    \centering
    \includegraphics[width=\linewidth]{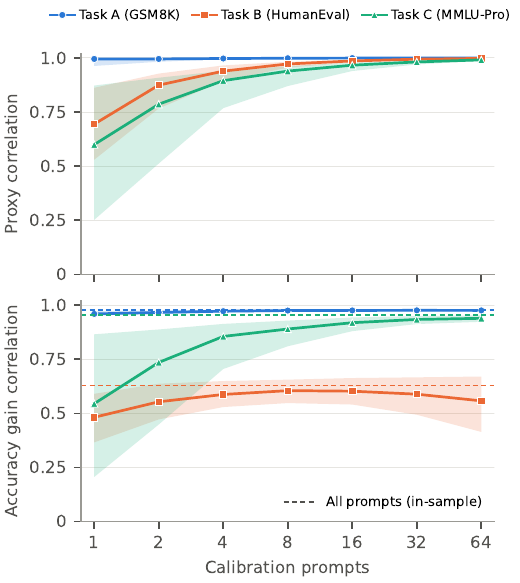}
    \caption{\textbf{A few unlabeled prompts suffice.} Proxy computed from $N$ random calibration prompts. \emph{Top:} correlation with the all-prompt proxy. \emph{Bottom:} correlation with held-out accuracy gains.}
    \label{fig:calib-size}
\end{figure}

\paragraph{LLaDA-1.5.}
LLaDA-1.5 is LLaDA-8B-Instruct further trained with variance-reduced preference optimization~\citep{zhu-etal-2026-llada}. Its architecture, tokenizer, sampler and EOS rules are unchanged. Figure~\ref{fig:kl_llada15} shows the resulting profiles, which closely match those of LLaDA-8B-Instruct (Figure~\ref{fig:temporal-sensitivity}).
The Pearson correlation between the two models' per-window profiles is 0.99 (GSM8K), 0.98 (HumanEval) and 0.92 (MMLU-Pro).
\begin{itemize}
    \item \textbf{GSM8K:} damage peaks in the first window and decays.
    \item \textbf{HumanEval:} damage rises to a broad peak over steps 40--64.
    \item \textbf{MMLU-Pro:} damage forms a broad early plateau, peaks at steps 16--23 and rises again in the last window.
\end{itemize}
LLaDA-1.5 is a fine-tuned LLaDA-8B, so this result shows that the temporal profile survives preference alignment.
It does not test other architectures; Dream addresses that below.

\begin{figure}[t]
    \centering \includegraphics[width=\columnwidth]{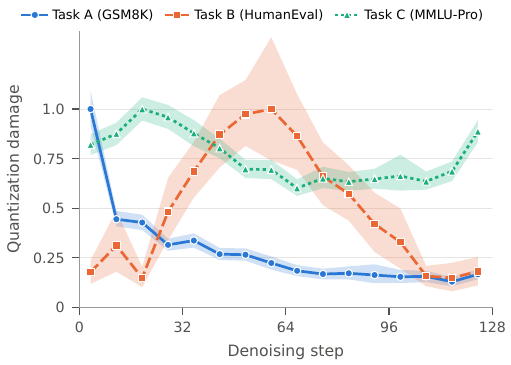}
    \caption{Label-free diagnostic for LLaDA-1.5 under W4A4, computed as in Figure~\ref{fig:temporal-sensitivity}: score-relevant KL per 8-step window, normalized to each task's peak, with 95\% problem-bootstrap bands.}
    \label{fig:kl_llada15}
\end{figure}

\paragraph{Dream-v0-Instruct-7B.}
Dream uses its own sampler~\citep{ye-etal-2025-dream}, which commits the lowest-entropy masked positions first.
Everything else follows the shared setup; only the handling of the end-of-sequence (EOS) token differs.
At the first denoising steps Dream predicts EOS at almost every masked position, so the EOS rule matters more than for LLaDA:
\begin{itemize}
    \item \textbf{No EOS rule (Dream's default):} padding is committed first, and the answer is squeezed into the last few steps.
    \item \textbf{LLaDA's rules:} the few non-EOS positions are forced in first regardless of their confidence.
    \item \textbf{Our choice:} the confidence of positions predicting EOS or end-of-turn is lowered by $\alpha = 1$ nat, on all tasks.
\end{itemize}
We chose $\alpha$, and entropy over top-2-margin ordering, among six variants by full-precision accuracy on the first 20 problems of each task.
The criterion was fixed before any damage profile was computed.

The profiles in Figure~\ref{fig:kl_dream} remain task-dependent, but the tasks peak at different points in denoising than on LLaDA.
\begin{itemize}
    \item \textbf{GSM8K:} Dream answers directly under the shared answer-only few-shot prompt (about six tokens), so the damage concentrates in the first window, with a small rise in the last one.
    \item \textbf{HumanEval:} damage stays low until step 64, then rises to a peak at steps 96--104, late rather than mid-trajectory as on LLaDA.
    \item \textbf{MMLU-Pro:} after a first-window spike (most W4A4 responses state the answer first), damage ramps up steadily to the end.
\end{itemize}

\begin{figure}[t]
    \centering \includegraphics[width=\columnwidth]{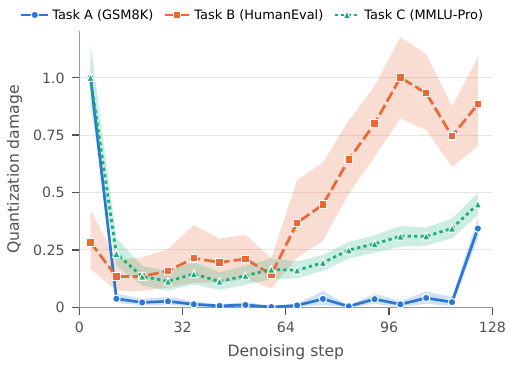}
    \caption{Label-free diagnostic for Dream-v0-Instruct-7B under W4A4, computed as in Figure~\ref{fig:temporal-sensitivity}: score-relevant KL per 8-step window, normalized to each task's peak, with 95\% problem-bootstrap bands.}
    \label{fig:kl_dream}
\end{figure}

\section{Additional Ablation Studies}
\label{app:ablations}

\begin{figure}[t]
\centering
\includegraphics[width=\columnwidth]{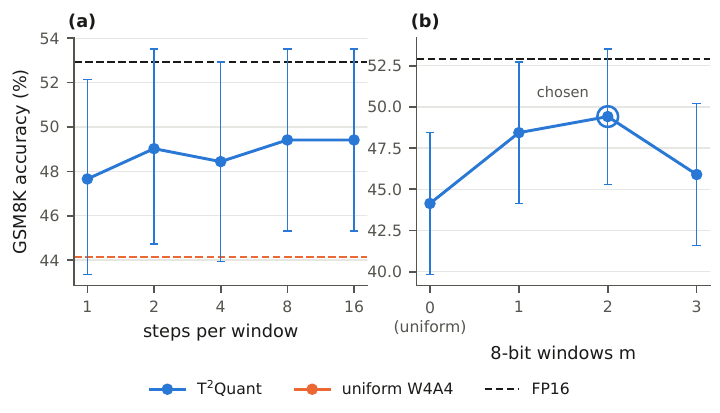}
\caption{Further ablations on LLaDA-8B-Instruct, GSM8K ($n=512$; bars are 95\% intervals): (a) steps per window, (b) number of 8-bit windows
$m$, the circle marking the allocation's own choice. Orange: uniform W4A4; dashed: FP16.}
\label{fig:ablations_more}
\end{figure}

\paragraph{Window size (Figure~\ref{fig:ablations_more}a).}
From 16 steps per window down to a single step the accuracy stays within 1.7 points (47.7 to 49.4), and the coarsest windows are as good as the
finest: the schedule needs few decisions, not many.

\paragraph{Number of 8-bit windows (Figure~\ref{fig:ablations_more}b).}
With three levels and a fixed 4.00-bit average, one 8-bit window is paid for by four 3-bit windows, so every schedule is described by the number
$m$ of 8-bit windows. Forcing $m=1$ or $m=3$ instead of the allocation's own choice $m=2$ costs 1.0 and 3.5 points, so the allocation lands on
the optimum of this one-parameter family. For every $m$ its schedule coincides with a calibration-free prior (8 bits on the first $m$ windows,
3 bits on the last $4m$); on GSM8K, choosing $m$ is where the calibration earns its keep.

\paragraph{Precision set (Table~\ref{tab:levelset_ablation}).}
Allowing all six widths $\{3,4,5,6,7,8\}$ instead of the four of the main table, $\{3,4,6,8\}$, changes the accuracy little: over the 36
composed cells of Table~\ref{tab:levelset_ablation}, the mean rises from 45.7 to 46.3, by 0.6 points on GSM8K, 0.1 on ARC-C and MMLU-Pro,
and 1.4 on HumanEval, whose 100-problem cells account for the largest single changes. A \method{} row is the best quantized result in 11 of the 12 model--benchmark cells. The method therefore does not depend on a particular choice of intermediate widths; we keep $\{3,4,6,8\}$ in the
main experiments because each additional width costs one more replay of the calibration trajectories.

\providecommand{\method}{T$^2$Quant}
\begin{table*}[t]
\centering\small\setlength{\tabcolsep}{5pt}
\caption{Precision-set ablation: W4A4 accuracy (\%) with \method{} allocating over the six activation widths $\{3,4,5,6,7,8\}$ instead of the four of Table~\ref{tab:main_results}, $\{3,4,6,8\}$, at the same 4.00 average activation bits. Avg.\ is over the four benchmarks shown (bold: best, underline: second-best quantized result).}
\label{tab:levelset_ablation}
\begin{tabular}{lccccc}
\toprule
\textbf{Method} & \textbf{Avg.} & \textbf{GSM8K} & \textbf{HumanEval} & \textbf{ARC-C} & \textbf{MMLU-Pro}\\
\midrule
\multicolumn{6}{l}{\textit{Dream-7B-Instruct}}\\
FP16 & 49.62 & 62.50 & 19.00 & 87.30 & 29.69\\
RTN & 0.00 & 0.00 & 0.00 & 0.00 & 0.00\\
DLLMQuant & 37.16 & 35.55 & 9.00 & 79.30 & 24.80\\
FAIR-Calib & 37.81 & 44.92 & 4.00 & 81.64 & 20.70\\
STaR-Quant & 41.57 & 46.09 & 11.00 & 83.01 & \underline{26.17}\\
Random schedule & 21.95 & 15.82 & 4.00 & 54.30 & 13.67\\
DLLMQuant + \method{} & \underline{45.35} & 53.52 & \underline{15.00} & \underline{84.18} & \textbf{28.71}\\
FAIR-Calib + \method{} & 44.61 & \underline{57.23} & 13.00 & 82.81 & 25.39\\
STaR-Quant + \method{} & \textbf{46.74} & \textbf{59.38} & \textbf{18.00} & \textbf{84.77} & 24.80\\
\midrule
\multicolumn{6}{l}{\textit{LLaDA-1.5-8B}}\\
FP16 & 50.74 & 57.03 & 26.00 & 82.62 & 37.30\\
RTN & 0.15 & 0.00 & 0.00 & 0.59 & 0.00\\
DLLMQuant & 40.42 & 35.94 & 15.00 & 78.52 & 32.23\\
FAIR-Calib & 45.27 & 47.07 & 18.00 & 80.86 & \underline{35.16}\\
STaR-Quant & 45.77 & 48.63 & \underline{20.00} & 79.49 & 34.96\\
Random schedule & 39.77 & 31.64 & \underline{20.00} & 76.56 & 30.86\\
DLLMQuant + \method{} & 43.72 & 47.07 & 12.00 & \textbf{83.79} & 32.03\\
FAIR-Calib + \method{} & \underline{48.58} & \underline{52.15} & \textbf{24.00} & \underline{83.01} & \underline{35.16}\\
STaR-Quant + \method{} & \textbf{49.85} & \textbf{57.23} & \textbf{24.00} & 82.23 & \textbf{35.94}\\
\midrule
\multicolumn{6}{l}{\textit{LLaDA-8B-Instruct}}\\
FP16 & 49.07 & 52.93 & 25.00 & 81.84 & 36.52\\
RTN & 0.34 & 0.39 & 0.00 & 0.98 & 0.00\\
DLLMQuant & 38.06 & 31.84 & 12.00 & 77.73 & 30.66\\
FAIR-Calib & 44.41 & 43.75 & 22.00 & 78.12 & \underline{33.79}\\
STaR-Quant & 45.36 & 45.70 & \textbf{25.00} & 78.52 & 32.23\\
Random schedule & 39.93 & 30.08 & \underline{23.00} & 75.39 & 31.25\\
DLLMQuant + \method{} & 43.42 & 43.36 & 18.00 & \textbf{81.84} & 30.47\\
FAIR-Calib + \method{} & \underline{46.12} & \underline{49.61} & 21.00 & \underline{80.66} & 33.20\\
STaR-Quant + \method{} & \textbf{47.95} & \textbf{50.98} & \textbf{25.00} & \underline{80.66} & \textbf{35.16}\\
\bottomrule
\end{tabular}
\end{table*}

\paragraph{Other precisions.}
Table~\ref{tab:abl_regimes} applies \method{} at two other budgets on LLaDA-1.5-8B and Dream-7B-Instruct. At W8A8 (8-bit weights, 8.00 average
activation bits over the widths $\{4,8,16\}$) uniform quantization is already at full precision on every cell, and the schedule is a harmless
no-op: one gain (LLaDA-1.5 GSM8K, +3.5, $p=0.04$), otherwise ties, pooled 132 / 123. At W4A3 (3.00 average bits) uniform quantization
collapses and the random schedule is catastrophic; the schedule recovers most of the loss, by up to 32 points on Dream ARC-C and 8.6 points on
LLaDA-1.5 GSM8K, pooled 371 / 133 with the widths $\{2,3,4\}$. The $\{2,3,8\}$ set is worse at this budget because one 8-bit window costs five
2-bit windows, and on two LLaDA-1.5 cells the allocation then declines to trade at all.

\paragraph{Proxy signal and aggregation (Table~\ref{tab:proxy_ablation}).}
On the window-upscaling experiments of Section~\ref{sec:analysis-signal}, we compare the proxy $S_w=\mathbb{E}[K_w]\,\mathbb{E}[I_w]$ with
three alternatives. The per-example form $\mathbb{E}[\sqrt{K_wI_w}]$ suggested by the proof ranks windows no better on any task and worse on
MMLU-Pro ($r=0.85$ vs.\ $0.96$). Divergence alone fails on HumanEval ($r=-0.08$). Influence alone ranks windows nearly as well as $S_w$, so most
of the ranking comes from $I_w$; the divergence term is what the allocation needs to price each precision level, which influence cannot do.

\begin{table}[h]
\centering\small\setlength{\tabcolsep}{4pt}
\caption{Label-free proxies vs.\ measured window gains (Figure~\ref{fig:proxy-correlation}): Pearson $r$ over the 16 windows.}
\label{tab:proxy_ablation}
\begin{tabular}{lccc}
\toprule
\textbf{Proxy} & \textbf{GSM8K} & \textbf{MMLU-Pro} & \textbf{HumanEval}\\
\midrule
$\mathbb{E}[K_w]\,\mathbb{E}[I_w]$ & 0.98 & 0.96 & 0.63\\
$\mathbb{E}[\sqrt{K_wI_w}]$ & 0.98 & 0.85 & 0.61\\
$\mathbb{E}[K_w]$ only & 0.97 & 0.68 & $-$0.08\\
$\mathbb{E}[I_w]$ only & 0.98 & 0.93 & 0.66\\
\bottomrule
\end{tabular}
\end{table}

\FloatBarrier

\section{Theoretical Proofs}
\label{app:proofs}

\subsection{Proof of Proposition~\ref{prop:bound}}
\label{app:bound-proof}

\paragraph{Notation.}
Fix an evaluation example $i$ and a denoising window $w$,
and let $s_w$ denote the state immediately before the window.
We suppress the superscript $(i)$ when unambiguous.
Let $\Pi_{\mathrm{up}}$ and $\Pi_{\mathrm{tg}}$ denote the
distributions over the complete commitment outcomes within
window $w$, conditioned on the common initial state $s_w$,
when the window is executed at
$\mathrm{A}_{\mathrm{up}}$ and
$\mathrm{A}_{\mathrm{target}}$, respectively.
A commitment outcome includes both the committed token
values and their positions, and therefore determines the
post-window state given $s_w$.

For a commitment outcome $c$, define
\[
g_w(c)
=
\mathbb{E}\left[
m(Y)\mid s_w,c
\right],
\]
where $m(Y)\in[0,1]$ is the final task score and the
expectation is over any randomness in the subsequent
$\mathrm{A}_{\mathrm{target}}$ continuation.
For a distribution $\Pi$ over commitment outcomes, let
\[
V_w(\Pi)
=
\mathbb{E}_{c\sim\Pi}[g_w(c)].
\]
The per-example effect of upscaling window $w$ is
\[
\Delta m_w
=
V_w(\Pi_{\mathrm{up}})
-
V_w(\Pi_{\mathrm{tg}}),
\]
and the average effect over evaluation examples is
\[
\Delta\mathcal{M}_w
=
\mathbb{E}_i[\Delta m_w^{(i)}].
\]

We define two theoretical quantities:
\begin{align*}
J_w
&=
D_{\mathrm{KL}}\left(
\Pi_{\mathrm{up}}
\,\Vert\,
\Pi_{\mathrm{tg}}
\right),\\
R_w
&=
\sup_c g_w(c)-\inf_c g_w(c).
\end{align*}
Here, $J_w$ measures the distributional change in
window commitments induced by upscaling, while $R_w$
measures the sensitivity of the expected downstream
task score to those commitments.

\paragraph{Perturbation--sensitivity bound.}
For stochastic commitment laws with finite $J_w$,
the expected task-score change satisfies
\[
|\Delta\mathcal{M}_w|^2
\le
\frac{1}{2}
\mathbb{E}_i[J_w^{(i)}]\,
\mathbb{E}_i[(R_w^{(i)})^2].
\]

\begin{proof}
Since $g_w(c)$ takes values in an interval of length
$R_w$, the variational characterization of total
variation gives
\[
\left|
V_w(\Pi_{\mathrm{up}})
-
V_w(\Pi_{\mathrm{tg}})
\right|
\le
R_w\,
\mathrm{TV}\left(
\Pi_{\mathrm{up}},
\Pi_{\mathrm{tg}}
\right).
\]
Applying Pinsker's inequality yields
\begin{align*}
|\Delta m_w|
&\le
R_w\sqrt{
\frac{1}{2}
D_{\mathrm{KL}}\left(
\Pi_{\mathrm{up}}
\,\Vert\,
\Pi_{\mathrm{tg}}
\right)
}\\
&=
\sqrt{\frac{1}{2}J_wR_w^2}.
\end{align*}
Averaging over examples and applying the triangle
inequality followed by Cauchy--Schwarz,
\begin{align*}
|\Delta\mathcal{M}_w|
&\le
\mathbb{E}_i[|\Delta m_w^{(i)}|]\\
&\le
\frac{1}{\sqrt{2}}
\mathbb{E}_i\left[
\sqrt{J_w^{(i)}}R_w^{(i)}
\right]\\
&\le
\sqrt{
\frac{1}{2}
\mathbb{E}_i[J_w^{(i)}]\,
\mathbb{E}_i[(R_w^{(i)})^2]
}.
\end{align*}
Squaring both sides establishes the result.
\end{proof}

\paragraph{From $J_w, R_w$ to $K_w, I_w$.}
Neither factor of the bound is available from a single baseline trajectory: $J_w$ compares the window's complete commitment laws, and
$R_w$ requires task scores under alternative commitments. Section~\ref{sec:analysis-signal} therefore uses $K_w$
(Eq.~\eqref{eq:kl}) and $I_w$ (Eq.~\eqref{eq:analysis-influence}) as surrogates. They are not bounds on $J_w$ and $R_w^2$: $K_w$ is
measured along the baseline trajectory rather than between commitment laws, and $I_w$ measures a shift in predictions rather than
task-score sensitivity. The proposition thus motivates $S_w$, but it neither bounds $\Delta\mathcal{M}_w$ by $S_w$ nor guarantees its
ranking, which we evaluate empirically.

\paragraph{Remarks.}
\emph{Deterministic decoding.} Under greedy decoding the commitment laws are point masses, so $J_w$ is infinite whenever the commitments
differ and the bound is vacuous; there the decomposition motivates the proxy without a quantitative guarantee.
\emph{Aggregation.} The product of task-level averages in the bound motivates $S_w$. The per-example alternative
$\mathbb{E}[\sqrt{K_w I_w}]$ can rank windows differently; $S_w$ performs as well or better in our comparisons
(Table~\ref{tab:proxy_ablation}).

\subsection{Proof of Proposition~\ref{prop:dp-optimal}}
\label{app:dp-proof}
For each $k$ and integer budget $z$, let $\mathcal A_k(z)$ denote the set of assignments to windows $1,\dots,k$ whose costs sum to $z$. We prove by induction on $k$ that $F_k(z)=\max_{\mathbf b\in\mathcal A_k(z)}\sum_{j=1}^k g_{\tau,j}(b_j)$, with the maximum of an empty feasible set equal to $-\infty$. For $k=0$, only the zero-cost empty assignment exists, matching the initialization. For $k>0$, every feasible assignment has a unique final choice $b\in\mathcal B$ with $c_k(b)\le z$, and its preceding $k-1$ windows consume exactly $z-c_k(b)$. By the induction hypothesis, the best possible value for that prefix is $F_{k-1}(z-c_k(b))$. Maximizing over all possible final choices gives precisely Eq.~\eqref{eq:dp}, proving the invariant. At $k=K$ and $z=C_0$, the invariant is the objective and constraint of Eq.~\eqref{eq:allocation}. Backtracking recovers an optimal assignment. \hfill $\square$

\begin{table*}[h]
\centering\footnotesize\setlength{\tabcolsep}{5pt}
\caption{Other precisions (LLaDA-1.5-8B and Dream-7B-Instruct; GSM8K / HumanEval / ARC-C, $n=512$ / 100 / 512): accuracy (\%), in parentheses the paired difference to the uniform row ($^{*}$: $p<0.05$); $^{=}$: the allocation returned the uniform schedule.}
\label{tab:abl_regimes}
\begin{tabular}{lrrrrrr}
\toprule
 & \multicolumn{3}{c}{\textbf{LLaDA-1.5-8B}} & \multicolumn{3}{c}{\textbf{Dream-7B-Instruct}}\\
\cmidrule(lr){2-4}\cmidrule(lr){5-7}
\textbf{Row} & GSM8K & HumanEval & ARC-C & GSM8K & HumanEval & ARC-C\\
\midrule
\multicolumn{7}{l}{\textit{W8A8, levels $\{4,8,16\}$, 8.00 bits}}\\
FP16 & 57.0 & 26.0 & 82.6 & 62.5 & 19.0 & 87.3\\
uniform W8A8 & 58.0 & 24.0 & 82.4 & 62.3 & 15.0 & 86.1\\
random schedule & 58.6 (+0.6) & 25.0 (+1.0) & 82.4 (+0.0) & 62.3 (+0.0) & 16.0 (+1.0) & 85.7 (-0.4)\\
\method{} & 61.5 (+3.5$^{*}$) & 24.0$^{=}$ & 83.0 (+0.6) & 60.4 (-2.0) & 14.0 (-1.0) & 85.9 (-0.2)\\
\midrule
\multicolumn{7}{l}{\textit{W4A3, 3.00 bits}}\\
FP16 & 57.0 & 26.0 & 82.6 & 62.5 & 19.0 & 87.3\\
uniform W4A3 & 30.3 & 12.0 & 76.8 & 5.3 & 1.0 & 46.9\\
random schedule & 12.1 (-18.2$^{*}$) & 0.0 (-12.0$^{*}$) & 57.6 (-19.1$^{*}$) & 1.2 (-4.1$^{*}$) & 0.0 (-1.0) & 20.5 (-26.4$^{*}$)\\
\method{}, levels $\{2,3,8\}$ & 30.3$^{=}$ & 12.0$^{=}$ & 81.2 (+4.5$^{*}$) & 4.5 (-0.8) & 1.0 (+0.0) & 78.9 (+32.0$^{*}$)\\
\method{}, levels $\{2,3,4\}$ & 38.9 (+8.6$^{*}$) & 13.0 (+1.0) & 79.5 (+2.7) & 9.6 (+4.3$^{*}$) & 1.0 (+0.0) & 77.5 (+30.7$^{*}$)\\
\bottomrule
\end{tabular}
\end{table*}

\section{Algorithms}
\label{app:algorithms}
Algorithms~\ref{alg:tq-calibration} and~\ref{alg:tq-dp} describe the calibration and allocation stages of \method. Following Section~\ref{sec:analysis-signal}, calibration combines quantization divergence at all committed positions with downstream influence. The candidate precisions and resource costs remain symbolic; their experimental values are given in Appendix~\ref{app:experimental_setup}.

For each calibration prompt $i$, denoising step $t$, and activation precision $b$, define the committed-position divergence
\begin{equation*}
\delta_{i,t}(b)=\sum_{j\in C_{i,t}}
D_{\mathrm{KL}}\!\left(p^{\mathrm{fp}}_{i,t,j}\Vert p^{(b)}_{i,t,j}\right).
\end{equation*}

\begin{algorithm}[t]
\caption{Label-free temporal calibration}
\label{alg:tq-calibration}
\begin{algorithmic}[1]
\Require Unlabeled task prompts $\mathcal D_\tau$; full-precision model $f_{\rm fp}$; calibrated quantized models $\{f_b:b\in\mathcal B\}$; reference precision $b_0$; $K$ temporal windows
\Ensure Task-specific precision scores $\{g_{\tau,k}(b)\}_{k,b}$
\State Initialize $H_{k,b}\gets0$, $Z_k\gets0$, $I_k\gets0$ for all $k,b$
\ForAll{$s_i\in\mathcal D_\tau$}
    \State Generate a reference trajectory with $f_{\rm fp}$
    \State Save states $x_{i,t}$, committed sets $C_{i,t}$, and $p^{\rm fp}_{i,t,j}$ for $j\in C_{i,t}$
    \For{$k=1,\ldots,K$}
        \State $Z_k\gets Z_k+\sum_{t\in\mathcal W_{i,k}}|C_{i,t}|$
        \State Obtain boundary states $x^-_{i,k}$, $x^+_{i,k}$ and remaining masked positions $M_{i,k}$
        \State Compute $\iota_{i,k}$ using Eq.~\eqref{eq:analysis-influence}
        \State $I_k\gets I_k+\iota_{i,k}$
        \ForAll{$t\in\mathcal W_{i,k}$}
            \ForAll{$b\in\mathcal B$}
                \State $p^{(b)}_{i,t,\cdot}\gets f_b(x_{i,t})$
                \State $H_{k,b}\gets H_{k,b}+\delta_{i,t}(b)$
            \EndFor
        \EndFor
    \EndFor
\EndFor
\For{$k=1,\ldots,K$}
    \State $\bar\iota_{\tau,k}\gets I_k/|\mathcal D_\tau|$
    \ForAll{$b\in\mathcal B$}
        \State $\kappa_{\tau,k}(b)\gets H_{k,b}/Z_k$
        \State $g_{\tau,k}(b)\gets\bar\iota_{\tau,k}\bigl[\kappa_{\tau,k}(b_0)-\kappa_{\tau,k}(b)\bigr]$
    \EndFor
\EndFor
\State \Return $\{g_{\tau,k}(b)\}_{k,b}$
\end{algorithmic}
\end{algorithm}

\noindent\textbf{Conventions.}
$\mathcal W_{i,k}$ denotes the steps assigned to normalized window $k$ for prompt $i$, and equals $\mathcal W_k$ when all prompts use the same step count. The reference trajectory and all quantized replays use identical saved input states, as in Section~\ref{sec:method-calibration}. We assume $Z_k>0$ for every window. If no masked positions remain after a window, its downstream influence is zero by Eq.~\eqref{eq:analysis-influence}. No ground-truth labels, task-specific evaluation metrics, or score-relevant position sets are used.

\begin{algorithm}[t]
\caption{Exact budget-constrained precision allocation}
\label{alg:tq-dp}
\begin{algorithmic}[1]
\Require Scores $g_{\tau,k}(b)$; nonnegative integer costs $c_k(b)$; precision set $\mathcal B$; reference precision $b_0$; $K$ windows; budget $C_0=\sum_{k=1}^{K}c_k(b_0)$
\Ensure Optimal schedule $\mathbf b_\tau^*$ for Eq.~\eqref{eq:allocation}
\State Initialize $F_k(z)\gets-\infty$, $P_k(z)\gets\bot$ for $0\le k\le K$ and $0\le z\le C_0$
\State $F_0(0)\gets0$
\For{$k=1,\ldots,K$}
    \For{$z=0,\ldots,C_0$}
        \ForAll{$b\in\mathcal B$ with $c_k(b)\le z$}
            \State $v\gets F_{k-1}(z-c_k(b))+g_{\tau,k}(b)$
            \If{$v>F_k(z)$}
                \State $F_k(z)\gets v$; $P_k(z)\gets b$
            \EndIf
        \EndFor
    \EndFor
\EndFor
\If{$F_K(C_0)\le 0$}
    \State \Return $(b_0,\ldots,b_0)$
\EndIf
\State $z\gets C_0$
\For{$k=K,K-1,\ldots,1$}
    \State $b_k^*\gets P_k(z)$
    \State $z\gets z-c_k(b_k^*)$
\EndFor
\State \Return $(b_1^*,\ldots,b_K^*)$
\end{algorithmic}
\end{algorithm}

The dynamic program requires $O(KC_0|\mathcal B|)$ time and $O(KC_0)$ memory, where $C_0$ is expressed in a common integer cost unit. Its optimality guarantee applies to the additive surrogate of Eq.~\eqref{eq:allocation}, not to downstream task accuracy. Rational costs can be converted into integer costs using a common exact scaling.

\section{Experimental Setup and Implementation Details}
\label{app:experimental_setup}

\subsection{Model, Sampler, and Quantization}
\label{app:shared_setup}
Unless stated otherwise, all experiments use LLaDA-8B-Instruct with 512-token generations over 128 denoising steps in a single
512-token block, so four positions are committed at each step. We use the standard low-confidence remasking sampler~\citep{nie-etal-2025-llada} with
temperature 0 and no classifier-free guidance, so decoding is deterministic. End-of-sequence handling follows the official
LLaDA evaluation settings for each benchmark.

Rotations follow QuaRot~\citep{ashkboos-etal-2024-quarot}: the residual stream is rotated by a random Hadamard matrix, and online Hadamard transforms are applied
before the attention output and MLP down projections. Weights of the transformer linear layers are quantized to 4 bits with
GPTQ~\citep{frantar-etal-2023-gptq} (per output channel, asymmetric, MSE-optimal clipping, act-order), calibrated on 128 WikiText~\citep{merity-etal-2016-wikitext} sequences of 512 tokens with
random masking. The embeddings and language-model head remain in high precision. Activations are quantized dynamically per
token, with a clipping ratio of 0.7, at the inputs to the QKV projection, the attention output projection, the MLP up/gate
projections, and the MLP down projection. Attention scores and softmax remain unquantized. We refer to the configuration with
4-bit weights and 4-bit activations as W4A4. The reference model uses BF16 computation, denoted fp16 throughout for
consistency. The random Hadamard matrix is drawn from a seed, which we refer to as the quantization seed: different seeds give different W4A4
models with the same configuration. The window-upscaling experiments of Section~\ref{sec:analysis-signal} average the measured gains
over five quantization seeds for GSM8K, ten for HumanEval and one for MMLU-Pro; all other experiments use a single seed.

\subsection{Main Results~(Section~\ref{sec:main_results})}
\label{app:main_setup}

\paragraph{Models.}
Dream-7B-Instruct~\citep{ye-etal-2025-dream} (\texttt{Dream-org/Dream-v0-Instruct-7B}), LLaDA-1.5~\citep{zhu-etal-2026-llada}
(\texttt{GSAI-ML/LLaDA-1.5}) and LLaDA-8B-Instruct~\citep{nie-etal-2025-llada} (\texttt{GSAI-ML/LLaDA-8B-Instruct}), with their released
weights and chat templates. All three are masked-diffusion LMs that decode a
fixed-length block by iteratively unmasking tokens; they differ in pre-training data, alignment recipe (LLaDA-1.5 adds preference
optimization) and, for Dream, in architecture and initialization, which makes Dream the most sensitive of the three to activation
quantization.

\paragraph{Benchmarks and splits.}
Table~\ref{tab:benchmarks} lists the benchmarks, the number of held-out problems, the number of calibration prompts, and the output
length. For each benchmark the problems are drawn once from the original evaluation set and split into a calibration pool and a held-out
evaluation pool; the calibration pool is the only data any calibrated method (the baselines and \method{}) sees, and the held-out pool is the
same for every row, so all comparisons in the paper are paired by problem. Accuracy is exact match of the extracted final answer after normalization (GSM8K, ARC-Challenge, MMLU-Pro, BBH; for DROP, exact match under the official DROP answer normalization against any annotated answer span, on the validation split, rather than the usual F1) or
pass@1 of the program on the benchmark's tests (HumanEval).
Averages are unweighted means over benchmarks.

\begin{table}[h]
\centering\small\setlength{\tabcolsep}{4pt}
\caption{Benchmarks: held-out problems, calibration prompts, output tokens and denoising steps (4 tokens per step).}
\label{tab:benchmarks}
\begin{tabular}{lrrrr}
\toprule
\textbf{Benchmark} & \textbf{Eval.} & \textbf{Calib.} & \textbf{Tokens} & \textbf{Steps}\\
\midrule
GSM8K & 512 & 256 & 512 & 128\\
HumanEval & 100 & 50 & 512 & 128\\
ARC-Challenge & 512 & 256 & 128 & 32\\
MMLU-Pro & 512 & 256 & 512 & 128\\
DROP & 512 & 256 & 64 & 16\\
BBH & 512 & 256 & 256 & 64\\
\bottomrule
\end{tabular}
\end{table}

\paragraph{Sampler.}
Every row uses the models' upstream sampler unchanged: 4 tokens are unmasked per step with low-confidence remasking, at temperature 0, with
the output lengths of Table~\ref{tab:benchmarks}. Rows therefore differ only in the quantization applied to the network.

\paragraph{Quantization setup.}
Quantization is simulated (the quantized integer values are used in bf16 arithmetic), the standard protocol for post-training quantization
studies~\citep{ashkboos-etal-2024-quarot, frantar-etal-2023-gptq}. Weights are quantized to 4 bits per channel with GPTQ after a Hadamard rotation of the hidden dimension (QuaRot-style, including
the online rotation of the attention-output input). Activations, the inputs of all linear layers, are quantized dynamically per token; for
widths below 8 bits the per-token range is clipped to 70\% of its maximum. This rotated, GPTQ-quantized W4 model with per-token A4
activations is the shared base of every quantized row except RTN, which uses round-to-nearest 4-bit per-channel weights and per-token A4
activations without rotation or GPTQ. \method{} keeps the base and varies only the activation width along the denoising trajectory: the
trajectory is divided into 16 windows, each window uses one of \{3, 4, 6, 8\} bits, and the allocation is chosen on the calibration pool under
an exact average of 4.00 bits (the method section gives the allocation). The random schedule assigns window widths at random under the same
4.00-bit average; one fixed draw (seed 0) is used for every model and benchmark.

\paragraph{Baselines.}
DLLMQuant~\citep{xu-etal-2025-dllmquant}, FAIR-Calib~\citep{huang-etal-2026-faircalib} and STaR-Quant~\citep{yan-etal-2026-starquant} are
reimplemented on the common rotated W4 base, so that every row shares the same weights, sampler and
prompts and only the activation treatment differs: DLLMQuant's calibration over sampled denoising states, FAIR-Calib's calibration prior, and
STaR-Quant's training-free state-aware channel scaling with its attention correction are each calibrated on denoising states produced by the
full-precision model from the calibration prompts. ``Base + \method{}'' applies our window allocation on top of the baseline's calibration,
which is left unchanged.

\end{document}